\documentclass{article}

 \usepackage[preprint]{neurips_2026}

\usepackage[utf8]{inputenc} 
\usepackage[T1]{fontenc}    
\usepackage{hyperref}       
\usepackage{url}            
\usepackage{booktabs}       
\usepackage{amsfonts}       
\usepackage{nicefrac}       
\usepackage{microtype}      
\usepackage{xcolor}         

\usepackage{graphicx}
\usepackage{bm}
\usepackage{amsmath}
\usepackage{algorithm}  
\usepackage{algpseudocode}
\usepackage{wrapfig}
\usepackage{multicol}
\usepackage{multirow}
\usepackage{subcaption}
\usepackage[table,xcdraw]{xcolor}
\usepackage{marvosym}

\usepackage{amsthm}

\newtheorem{lemma}{Lemma}

\newtheorem{definition}{Definition}

\usepackage{tcolorbox}
\tcolorboxenvironment{proof}{
  colback=gray!8,
  boxsep=0pt,
  left=8pt, right=8pt, top=8pt, bottom=8pt,
  boxrule=1pt
}

\title{Control Your View: High-Resolution Global Semantic Manipulation in Learned Image Compression}

\author{%
  Jiaming Liang
  \\
  University of Macau\\
  Macau, China \\
  \texttt{chinaliangjm@gmail.com} \\
  \And
  Chi-Man Pun \Letter\\
  University of Macau \\
  Macau, China \\
  \texttt{cmpun@um.edu.mo} \\
  \AND
  Weisi Lin \\
  Nanyang Technological University \\
  Singapore \\
  \texttt{wslin@ntu.edu.sg} \\
  \And
  Greta Seng Peng Mok \Letter\\
  University of Macau \\
  Macau, China \\
  \texttt{gretamok@um.edu.mo} \\
}

\definecolor{darkblue}{RGB}{0,0,100}
\definecolor{darkgreen}{RGB}{0,100,0}
\definecolor{revise}{RGB}{255,0,0}
\begin{document}

\maketitle

\begin{abstract}
Learned image compression (LIC) integrates deep neural networks (DNNs) to map high-dimensional images into compact latent representations, reducing redundancy and achieving superior rate–distortion (RD) performance in benign settings. Unfortunately, due to inherent vulnerabilities in DNNs, LIC systems are susceptible to adversarial perturbations that lead to downstream deterioration, compression rate degradation, untargeted distortion, and both local semantic manipulation (LSM) and low-resolution ($3\times28\times28$) global semantic manipulation (GSM). 
However, high-resolution GSM remains unexplored due to its intractability. Notably, the existing project gradient descent (PGD) method achieves near-perfect white-box attacks for classification, segmentation, and other tasks, yet fails to generalize to high-resolution GSM. 
Our theoretical and empirical analyses reveal that well-performing GSM drives adversarial examples from the Identity Region to the Amplification Region through the Lazying–Oscillating–Refining stages. 
General $\ell_{\infty}$-bounded attacks fail on high-resolution GSM because their step-size schedules cannot accommodate both the Oscillating and Refining stages. 
Based on this, we propose the Periodic Geometric Decay schedule that enables $\ell_{\infty}$-bounded high-resolution GSM. 
To verify our approach, we integrate it with PGD, yielding a minimal variant, PGD$^{2}$-GSM. 
Extensive experiments on the Kodak $(3\times768\times512)$ demonstrate that our PGD$^{2}$-GSM is the first to stably achieve high-resolution GSM, thereby exposing a novel threat to LIC systems. Code is available at \url{https://github.com/chinaliangjiaming/PGD2-GSM}.
\end{abstract}

\section{Introduction}
\begingroup
\renewcommand{\thefootnote}{}
\footnotetext{\Letter \; Corresponding authors.}
\endgroup
\begin{figure}[h]
    \centering
    \includegraphics[width=0.92\linewidth]{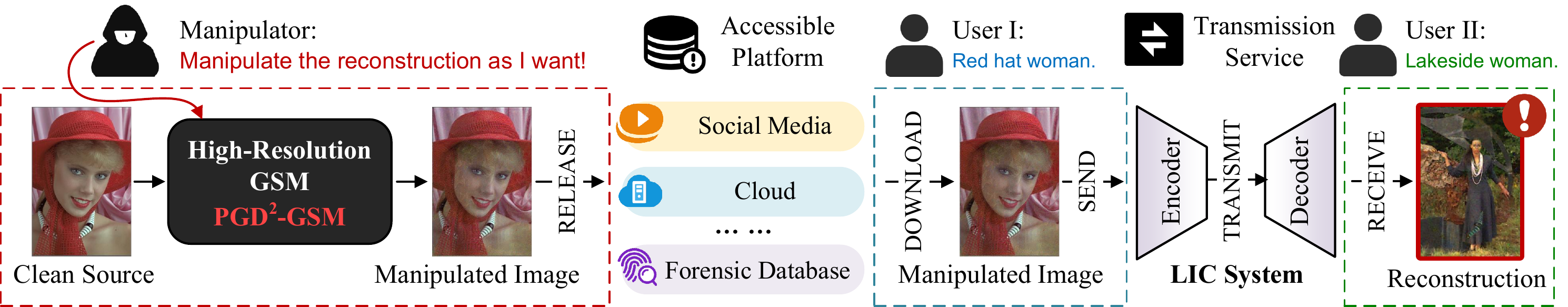}
    \caption{Illustrative threat scenario of high-resolution GSM: Manipulator controls the reconstruction perceived by User II after the LIC encoder-decoder, while both User I and User II remain unaware of the malicious manipulation. The stealthiness and controllability of GSM pose significant risks.}
    \label{fig:high_resolution_GSM_threat_case}
\end{figure}

\begin{figure}
    \centering
    \includegraphics[width=0.80\linewidth]{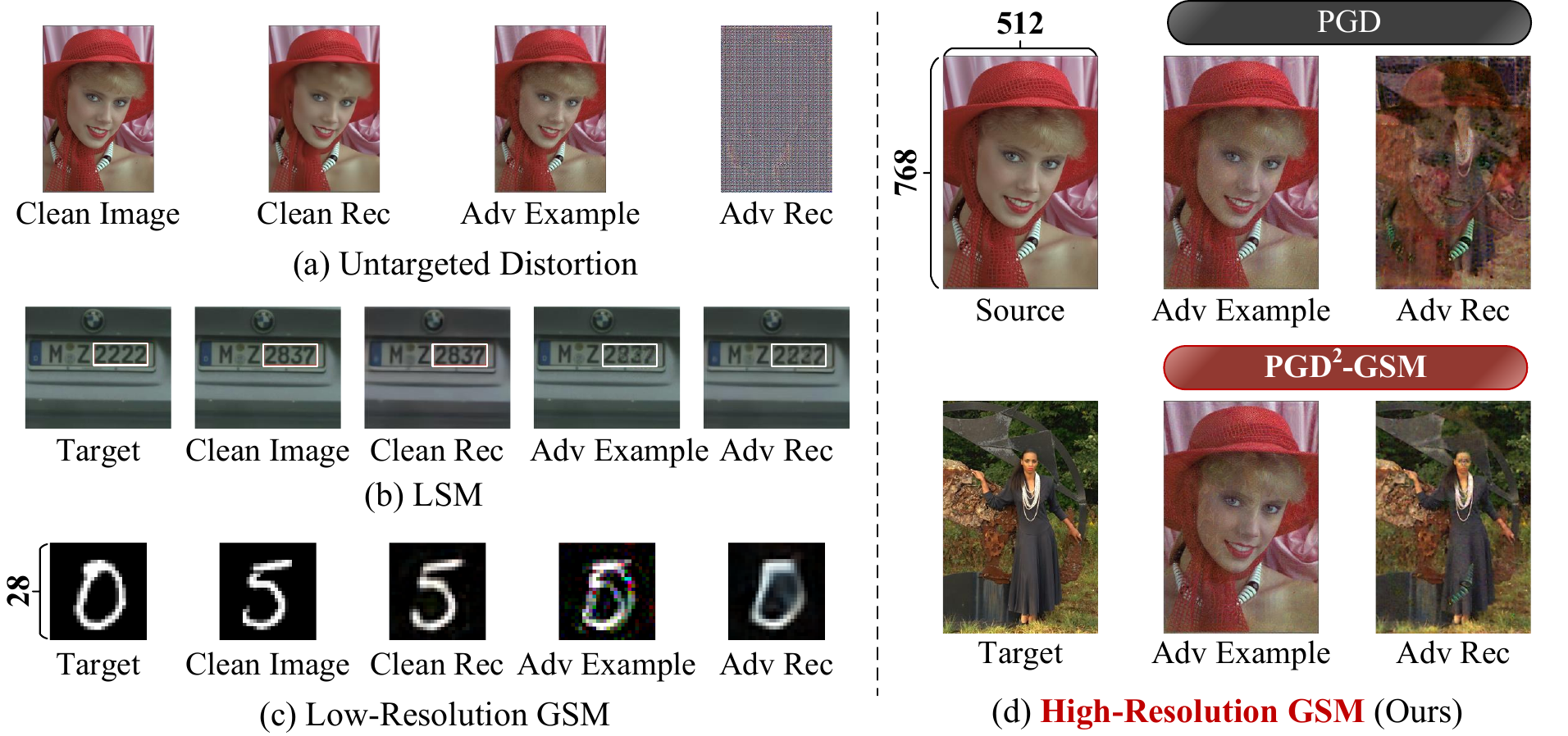}
    \caption{Overview of reconstruction fidelity attacks. Rec and Adv are abbreviations for reconstruction and adversarial, respectively. (a) Untargeted distortion induced by PGD. 
    (b) LSM from~\cite{chen2023toward}, with the manipulated region highlighted by a white box. (c) Low-resolution GSM from~\cite{chen2023toward}. (d) Our high-resolution GSM. The proposed minimal yet effective PGD$^{2}$-GSM enables stable $\ell_{\infty}$-bounded high-resolution GSM in LIC systems via our Periodic Geometric Decay step-size schedule. 
    }
    \label{fig:teaser}
\end{figure}
\textbf{Background.} Image compression reduces the bit rate for efficient storage and transmission, conditioned on reconstruction fidelity constraints. Despite significant progress, conventional hand-crafted codecs such as JPEG~\cite{wallace1991jpeg}, JPEG 2000~\cite{skodras2002jpeg}, HEVC/H.265~\cite{sullivan2012overview}, and VVC/H.266~\cite{bross2021overview} are still limited by their insufficient exploitation of redundancies. Recently, learned image compression (LIC) has harnessed the expressive nonlinear representation capacity of deep neural networks (DNNs) to jointly learn analysis–synthesis transforms~\cite{balle2017end} and entropy models~\cite{balle2018variational, minnen2018joint} end-to-end, overcoming the limitations of hand-crafted designs and significantly improving rate–distortion (RD) performance. Despite the superior RD performance of LIC systems under benign settings~\cite{fu2024weconvene, han2024causal, jiang2024llic, feng2025linear, lu2025learned, li2025learned, duan2026learned, zhang2026qarv++}, the incorporation of DNNs inevitably exposes the systems to adversarial threats~\cite{szegedy2013intriguing, goodfellow2015explaining}, especially given the white-box accessibility of LIC models in practical deployments, thereby attracting growing attention to LIC security. Existing studies on LIC threats primarily fall into three branches: one focuses on impairing downstream tasks~\cite{sui2024transferable}, another seeks to degrade compression rate~\cite{liu2023manipulation, yu2023backdoor, wu2025adversarial, kurihara2025efficient}, and the third investigates disrupting reconstruction fidelity~\cite{chen2023toward, yu2023backdoor, sui2024reconstruction, ma2024imperceptible, wu2025adversarial, kurihara2025efficient, kalmykov2026t}. 

Intriguingly, apropos of reconstruction fidelity, research~\cite{chen2023toward} has demonstrated that adversarial perturbations can not only induce untargeted distortion in reconstructed images (Figure~\ref{fig:teaser}(a)), but can even steer them toward attacker-specified targets, i.e., semantic manipulation (Figure~\ref{fig:teaser}(b-d)). Based on the spatial scope of the manipulated reconstruction, semantic manipulation in LIC can be divided into local semantic manipulation (LSM) and global semantic manipulation (GSM). LSM (Figure~\ref{fig:teaser}(b)) and low-resolution GSM (Figure~\ref{fig:teaser}(c), validated in~\cite{chen2023toward} on three-channel MNIST handwritten digit images of size $28\times28$) are readily attainable through $\ell_{\infty}$-bounded PGD~\cite{madry2018towards} or $\ell_{2}$-bounded C\&W~\cite{carlini2017towards} attacks owing to the limited search space. However, a more concerning form of \textcolor{darkblue}{high-resolution GSM} (Figure~\ref{fig:teaser}(d)) remains unexplored, despite its potential to fundamentally compromise the LIC security, as illustrated by a representative threat scenario in Figure~\ref{fig:high_resolution_GSM_threat_case}. To bridge this gap, we systematically investigate high-resolution GSM for the first time. Since $\ell_{\infty}$-bounded attacks exhibit higher efficiency than the $\ell_{2}$-bounded attacks in LIC systems~\cite{chen2023toward}, we focus on $\ell_{\infty}$-bounded high-resolution GSM here.

\textbf{Challenges}. Counterintuitively, the strong white-box attack PGD performs near-perfectly for classification~\cite{madry2018towards}, segmentation~\cite{gu2022segpgd}, and other tasks~\cite{zhang2019towards}, yet fails in high-resolution GSM, as illustrated in Figure~\ref{fig:teaser}(d). 
The challenges arise from two aspects: (1) High-resolution images induce a prohibitively large search space, requiring $C\times H\times W$ perturbations to satisfy equally many reconstruction constraints. (2) More essentially, unlike most deep learning tasks with heterogeneous input–output mappings that make it easier for perturbations to identify directions, LIC models exhibit an intrinsic identity-mapping property, with outputs approximating inputs for most samples. We find that this flatter mapping further increases the difficulty of high-resolution GSM. Despite these challenges, we identify the causes of general $\ell_{\infty}$-bounded attacks' failure and adapt them into the effective variants.

\textbf{Theory and Methodology}. Our analysis (Section~\ref{section: Identity Region and Amplification Region}) shows that the input space of well-trained LIC models splits into a wide \textcolor{darkblue}{Identity Region} and a narrow \textcolor{darkblue}{Amplification Region}. Adversarial examples initially reside in the identity region, where the identity-mapping property drives perturbations toward the $\epsilon$-bounded target–source difference, forming \textcolor{darkblue}{Lazy Examples}. Since effective GSM examples exist only in the amplification region, transitioning from the identity region to the amplification region is critical. Once in the amplification region, they must be refined to locate operating points that amplify perturbations toward the target-source difference. Accordingly, we characterize effective GSM as \textcolor{darkblue}{Lazying-Oscillating-Refining}, and introduce \textcolor{darkblue}{Lazy Cosine Similarity} for validation (Section~\ref{section: Lazying-Oscillating-Refining}). 

In the Oscillating stage, a large step size $\alpha$ is needed to explore a broader region and rapidly approach the amplification region. In the Refining stage, the large step size $\alpha$ destabilizes adversarial examples and may push them out of the amplification region. Therefore, a smaller $\alpha$ is required for precise localization. General $\ell_{\infty}$-bounded attacks fail due to their step-size schedules cannot simultaneously accommodate the Oscillating and Refining stages (Section~\ref{section: PGD2_GSM}). 
Therefore, we propose the \textcolor{darkblue}{Periodic Geometric Decay} schedule to enable $\ell_{\infty}$-bounded attacks on high-resolution GSM. To verify our approach, we integrate it with PGD, yielding a minimal variant, \textcolor{darkblue}{PGD$^{2}$-GSM}. 
Extensive experiments (Section~\ref{section: experiments}) on the Kodak ($3\times768\times512$) demonstrate that even this minimal variant effectively enables stable high-resolution GSM for the first time. Our contributions are summarized as follows:
\begin{itemize}
    \item To the best of our knowledge, this is the first study of high-resolution GSM in LIC. The stealthiness and controllability of high-resolution GSM poses a more formidable challenge.
    \item We identify that the input space of LIC models partitions into the Identity Region and the Amplification Region, with effective GSM adversarial examples residing in the latter. Accordingly, we further deduce that well-performing GSM undergoes three stages: Lazying-Oscillating-Refining, and propose Lazy Cosine Similarity to validate this process.
    \item Building on this progression, we find that effective GSM requires a large step size early to broaden exploration and quickly enter the amplification region, followed by a smaller step size to locate suitable operating points within it. Accordingly, we propose the Periodic Geometric Decay step-size schedule, yielding the minimal yet effective variant PGD$^{2}$-GSM.
    \item Extensive experiments on the high-resolution Kodak dataset demonstrate that our PGD$^{2}$-GSM stably achieves $\ell_{\infty}$-bounded high-resolution GSM in LIC systems for the first time.
\end{itemize}

\section{Related Work}
\subsection{Learned Image Compression}
LIC follows the transform coding framework, consisting of transform, quantization, and entropy coding. Unlike conventional codecs, it replaces hand-crafted transforms with learnable analysis–synthesis transforms~\cite{balle2017end}, incorporating entropy modeling to estimate the latent distribution and minimize code length~\cite{balle2018variational, minnen2018joint}, and employing differentiable quantization for end-to-end optimization~\cite{balle2017end}. Extensive research has been devoted to optimizing the components of LIC, yielding significant improvements. \textbf{Transform Optimizations.} To improve RD performance, various analysis–synthesis transform architectures have been developed. \textit{STF}~\cite{zou2022devil} introduces window-based attention blocks. \textit{NVTC}~\cite{feng2023nvtc} introduces nonlinear vector transforms. \textit{TCM}~\cite{liu2023learned} proposes mixed Transformer–CNN architectures. \textit{FTIC}~\cite{li2024frequency} integrates a frequency-aware Transformer. \textit{LLIC}~\cite{jiang2024llic} introduces large-kernel depth-wise convolutions and self-conditioned weight generation. \textit{LALIC}~\cite{feng2025linear} incorporates linear attention blocks. \textit{QARV++}~\cite{zhang2026qarv++} integrates disentangled latent mappings and deformable convolutions. \textbf{Quantization Optimizations.} Standard differentiable approximations for quantization include \textit{Stochastic Quantization}~\cite{toderici2016variable}, \textit{Uniform Noise Injection}~\cite{balle2017end}, \textit{Straight-Through Estimator}~\cite{theis2017lossy}, and hybrid approaches~\cite{minnen2020channel}. To enhance both RD performance and training efficiency, especially for variable-rate models, a variety of quantization surrogates have been developed, such as \textit{Universal Quantization}~\cite{agustsson2020universally}, \textit{Dead-Zone Quantizer}~\cite{zhou2020variable}, \textit{Soft Bit}~\cite{alexandre2019learned}, \textit{Soft‑then‑Hard}~\cite{guo2021soft}, and \textit{STanH}~\cite{presta2025stanh}. 
\textbf{Entropy Coding Optimizations.} To more accurately estimate the latent distribution and improve coding efficiency, various entropy modeling techniques have been developed.~\cite{minnen2020channel} introduces channel-conditioning and latent residual prediction. \textit{CCM}~\cite{he2021checkerboard} proposes a parallelizable checkerboard context model with two-pass decoding. \textit{ELIC}~\cite{he2022elic} explores an uneven grouping strategy to accelerate the channel-conditional method and integrates it with a parallel spatial context model. \textit{Entroformer}~\cite{qian2022entroformer} proposes a parallel bidirectional context model. 
\textit{HPCM}~\cite{li2025learned} proposes the hierarchical coding and the progressive context fusion. Many other interesting works~\cite{li2024multirate, kim2024diversify, perugachi2024robustly, fu2024weconvene, han2024causal, zhang2024learning, jiang2025mlic++, duan2026learned, wang2026distilling} are not detailed for brevity.

\subsection{Adversarial Threats on Learned Image Compression}
Despite continual advances in RD performance, LIC models inherently suffer from adversarial threats. Owing to their distinctive identity-mapping property (Section~\ref{section: Identity Region and Amplification Region}) compared to other deep models, LIC security merits dedicated investigation. Unfortunately, research in this area remains scant.~\cite{chen2023toward} represents the first systematic study on the LIC security, considering both $\ell_{2}$ and $\ell_{\infty}$ attacks. Its robustness evaluation focuses on reconstruction fidelity, including untargeted distortion, LSM, and low-resolution GSM. However, this work relies on basic \textit{I-FGSM}~\cite{goodfellow2015explaining}, \textit{C\&W}~\cite{carlini2017towards}, and a simple variant of \textit{C\&W}, making high-resolution GSM remains highly challenging.~\cite{yu2023backdoor} proposes the first backdoor attack specifically targeting LIC models. The designed frequency-based trigger is shown to degrade compression rates, distort reconstructed images, and mislead downstream tasks. To enhance the imperceptibility, inspired by the human visual system,~\cite{ma2024imperceptible} constrains the perturbations to the high-frequency regions of the chrominance components in YUV images. Around the same time,~\cite{sui2024reconstruction} leverages the frequency-domain representation induced by the discrete cosine transform to improve perturbation imperceptibility.~\cite{wu2025adversarial} presents the first study to evaluate joint rate–distortion attacks on LIC models.~\cite{madden2025bitstream} first proposes and demonstrates the effectiveness of bitstream collision attacks. T-MLA~\cite{kalmykov2026t} proposes a multiscale log–exponential frequency-aware attack framework to enhance imperceptibility in both spatial and spectral domains. However, high-resolution GSM remains unexplored. It fundamentally undermines LIC reliability by enabling reconstructions to be dictated by malicious intent. We present the first systematic study on this threat.

\section{Theory and Methodology}
\subsection{Preliminaries}
We begin by presenting an abstract framework of LIC that accommodates diverse transform architectures, quantization strategies, and entropy modeling approaches. A complete LIC model $f$ can be expressed as a cascade of an encoder $E$, a quantizer $Q$, a dequantizer $Q^{-1}$, and a decoder $D$:
\begin{equation}
    f = D \circ \ Q^{-1} \circ Q \circ E,
\end{equation}
where $\circ$ denotes the composition of functions. During the encoding stage, the encoder $E$ maps the input image $\boldsymbol{x}\in \mathcal{X}$ from the sample space $\mathcal{X}$ to a latent representation $\boldsymbol{y}$ and a hyper latent $\boldsymbol{z}$ that captures distributional information, which are subsequently quantized by $Q$ into discreted $\hat{\boldsymbol{y}}$ and $\hat{\boldsymbol{z}}$:
\begin{equation}
    \left[\hat{\boldsymbol{y}}, \hat{\boldsymbol{z}}\right] = Q\left(E\left(\boldsymbol{x}\right)\right).
\end{equation}
During the decoding stage, the dequantizer $Q^{-1}$ recovers the noisy latent representation $\tilde{\boldsymbol{y}}$ from $\hat{\boldsymbol{y}}$ and $\hat{\boldsymbol{z}}$, which are then fed into the decoder $D$ to produce the reconstructed image $\hat{\boldsymbol{x}}$:
\begin{equation}
    \hat{\boldsymbol{x}} = D(Q^{-1}\left(\hat{\boldsymbol{y}}, \hat{\boldsymbol{z}}\right)).
\end{equation}
The objective of a LIC model is to balance bit rate $\mathcal{R}$ and distortion $\mathcal{D}$ via a Lagrange multiplier $\lambda$:
\begin{equation}
    \min \mathcal{L}=\mathbb{E}_{\boldsymbol{x}\sim\mathcal{X}}[\underbrace{r\left(\hat{\boldsymbol{y}}\right) + r\left(\hat{\boldsymbol{z}}\right)}_{\mathcal{R}} + \lambda\cdot\underbrace{d\left(\boldsymbol{x}, \hat{\boldsymbol{x}}\right)}_{\mathcal{D}}].
\end{equation}
Here, $r\left(\cdot\right)$ measures the bit rate, and $d\left(\cdot,\cdot\right)$ is a perceptual distortion measure.

Compression models are generally available as universal standards, with their architectures and parameters fully accessible~\cite{kos2018adversarial, chen2023toward}. Consequently, attackers typically operate in white-box settings. Under this scenario, $\ell_{\infty}$-bounded GSM is defined as follows:
\begin{definition}
    \textbf{$\ell_{\infty}$-bounded GSM.} Given an accessible LIC model $f$ and a set of victim-target pairs $S=\{(\boldsymbol{x}_{p_{i}}, \boldsymbol{x}_{q_{i}})\mid\boldsymbol{x}_{p_{i}}, \boldsymbol{x}_{q_{i}}\in \mathcal{X},0\leq i< N\}$, determine the set of perturbations $\{\boldsymbol{\delta}_{i}\mid||\boldsymbol{\delta}_{i}||_{\infty}\leq\epsilon, 0\leq i< N\}$, minimize the following objective:
\begin{equation}
    \arg\min_{\{\boldsymbol{\delta}_{i}\}} \mathbb{E}_{0\leq i< N}\left[d(f(\boldsymbol{x}_{p_{i}}+\boldsymbol{\delta}_{i}), \boldsymbol{x}_{q_{i}})\right].
    \label{GSM_pairs_optimization_objective}
\end{equation}
\end{definition}
We consider the general Mean Squared Error (MSE) as $d\left(\cdot,\cdot\right)$ here. Substituting MSE into Equation~\ref{GSM_pairs_optimization_objective} for a pair $(\boldsymbol{x}_{p}, \boldsymbol{x}_{q})$ of shape $(C, H, W)$ and scaling by $\frac{C\times H\times W}{2}$, the single-pair objective is obtained:
\begin{equation}
    \arg\max_{\boldsymbol{\delta}} \phi=-\frac{1}{2}||f(\boldsymbol{x}_{p}+\boldsymbol{\delta})-\boldsymbol{x}_{q}||^{2}_{2}.
    \label{GSM_single_pair_optimization_objective}
\end{equation}
To achieve this, the strong white-box attack PGD~\cite{madry2018towards} initializes $\boldsymbol{\delta}$ as $\boldsymbol{\delta^{(0)}}\sim\text{Uniform}(-\epsilon,\epsilon)$ and iteratively performs projected gradient descent for $T$ steps, with the recursive expression:
\begin{equation}
    \boldsymbol{\delta}^{(t+1)}=\Pi_{[-\epsilon,\epsilon]}(\boldsymbol{\delta}^{(t)}+\alpha\cdot\mathrm{sgn}\left( \nabla_{\boldsymbol{\delta}^{(t)}}\phi\right)),\quad0\leq t<T.
    \label{recursive_PGD}
\end{equation}
Here, $\mathrm{sgn}(\cdot)$ is the sign function, and $\Pi_{[-\epsilon,\epsilon]}(\cdot)$ projects the perturbation onto the $\epsilon$-ball. Surprisingly, even such a highly effective $\ell_{\infty}$-bounded attack on other tasks fails to generalize to high-resolution GSM. 
We next elucidate the causes and present an \textit{indispensable} strategy for high-resolution GSM.

\subsection{Identity Region and Amplification Region}
\label{section: Identity Region and Amplification Region}


For most tasks, the exact form of $\nabla_{\boldsymbol{\delta}^{(t)}}\phi$ is complicated due to the heterogeneity between the input and output. In contrast, a well-trained LIC model $f$ enforces $\hat{\boldsymbol{x}}\approx\boldsymbol{x}$, thereby exhibiting an approximate identity-mapping property for most benign data points. Based on the identity-mapping property, we decompose $f(\boldsymbol{x})$ into the input $\boldsymbol{x}$ and the reconstruction residual $\boldsymbol{\eta}$ induced by $f$:
\begin{equation}
    f(\boldsymbol{x})=\boldsymbol{x}+\boldsymbol{\eta},
    \label{approximation_f}
\end{equation}
we have $\eta=||\boldsymbol{\eta}||\ll||\boldsymbol{x}||$ in benign regions. Therefore, $\nabla_{\boldsymbol{\delta}^{(t)}}\phi$ can in fact be approximated as:
\begin{equation}
\nabla_{\boldsymbol{\delta}^{(t)}}\phi=-\left[f(\boldsymbol{x}_{p}+\boldsymbol{\delta}^{(t)})-\boldsymbol{x}_{q}\right]\cdot\nabla_{\boldsymbol{\delta}^{(t)}}f(\boldsymbol{x}_{p}+\boldsymbol{\delta}^{(t)}))\approx-\left[(\boldsymbol{x}_{p}+\boldsymbol{\delta}^{(t)}+\boldsymbol{\eta}^{(t)})-\boldsymbol{x}_{q}\right]\cdot\boldsymbol{I},
\label{approxiamation_gradient_1}
\end{equation}
for benign regions. Due to the imperceptibility of perturbations and the generalization capability of the LIC model, both $||\boldsymbol{\delta}^{(t)}||$ and $||\boldsymbol{\eta}^{(t)}||$ are negligible compared to $||\boldsymbol{x}_{q}-\boldsymbol{x}_{p}||$. These terms and the approximation error are subsumed into a negligible term $\boldsymbol{o}^{(t)}$. Accordingly, we have:
\begin{equation}
\nabla_{\boldsymbol{\delta}^{(t)}}\phi=\boldsymbol{x}_{q}-\boldsymbol{x}_{p}-\boldsymbol{o}^{(t)},\quad||\boldsymbol{o}^{(t)}||\ll||\boldsymbol{x}_{q}-\boldsymbol{x}_{p}||.
\label{approxiamation_gradient_2}
\end{equation}
Equation~\ref{approxiamation_gradient_2} suggests that, within the benign region, PGD iteratively updates $\boldsymbol{x}_{p}+\boldsymbol{\delta}^{(t)}$ toward the \textcolor{darkblue}{Lazy Example} $\boldsymbol{x}_{p}+\Pi_{[-\epsilon,\epsilon]}(\boldsymbol{x}_{q}-\boldsymbol{x}_{p})$ following Equation~\ref{recursive_PGD}, until saturation. We refer to $\Pi_{[-\epsilon,\epsilon]}(\boldsymbol{x}_{q}-\boldsymbol{x}_{p})$ as the \textcolor{darkblue}{Lazy Perturbation}. Section~\ref{section: Lazying-Oscillating-Refining} will show that this process is inevitable in the initial stage.


The above analysis and conclusions are established in the regions where the identity-mapping property consistently holds. We now formally define these regions and their complement relative to $\mathcal{X}$:
\begin{definition}
    \textbf{Identity Region $R_{I}$ and Amplification Region $R_{A}$.} For a well-trained LIC model $f$, the \textcolor{darkblue}{Identity Region} $R_{I}=\{\boldsymbol{x}\in\mathcal{X}\mid \eta\leq \eta_{0}\ll||\boldsymbol{x}_{q}-\boldsymbol{x}_{p}||\}$, where $\eta_{0}$ denotes the upper bound of the deviation. The complement of $R_{I}$ is defined as the \textcolor{darkblue}{Amplification Region} $R_{A}=\mathcal{X}\setminus R_{I}$.
    \label{definition_identity_region}
\end{definition}
From Definition~\ref{definition_identity_region}, the location of effective GSM examples follows (see appendix~\ref{proof:adv_in_amplification_region} for proof):

\begin{lemma}
    Effective GSM adversarial example $\boldsymbol{x}_{p}+\boldsymbol{\delta}$ lies in amplification region $R_{A}$.
    \label{lemma:adv_in_amplification_region}
\end{lemma}

\subsection{Three-Stage Progression of GSM: Lazying-Oscillating-Refining}
\label{section: Lazying-Oscillating-Refining}

\begin{wrapfigure}{r}{0.42\columnwidth}
\vspace{-12pt}
\centering
\footnotesize
\includegraphics[width=0.98\linewidth]{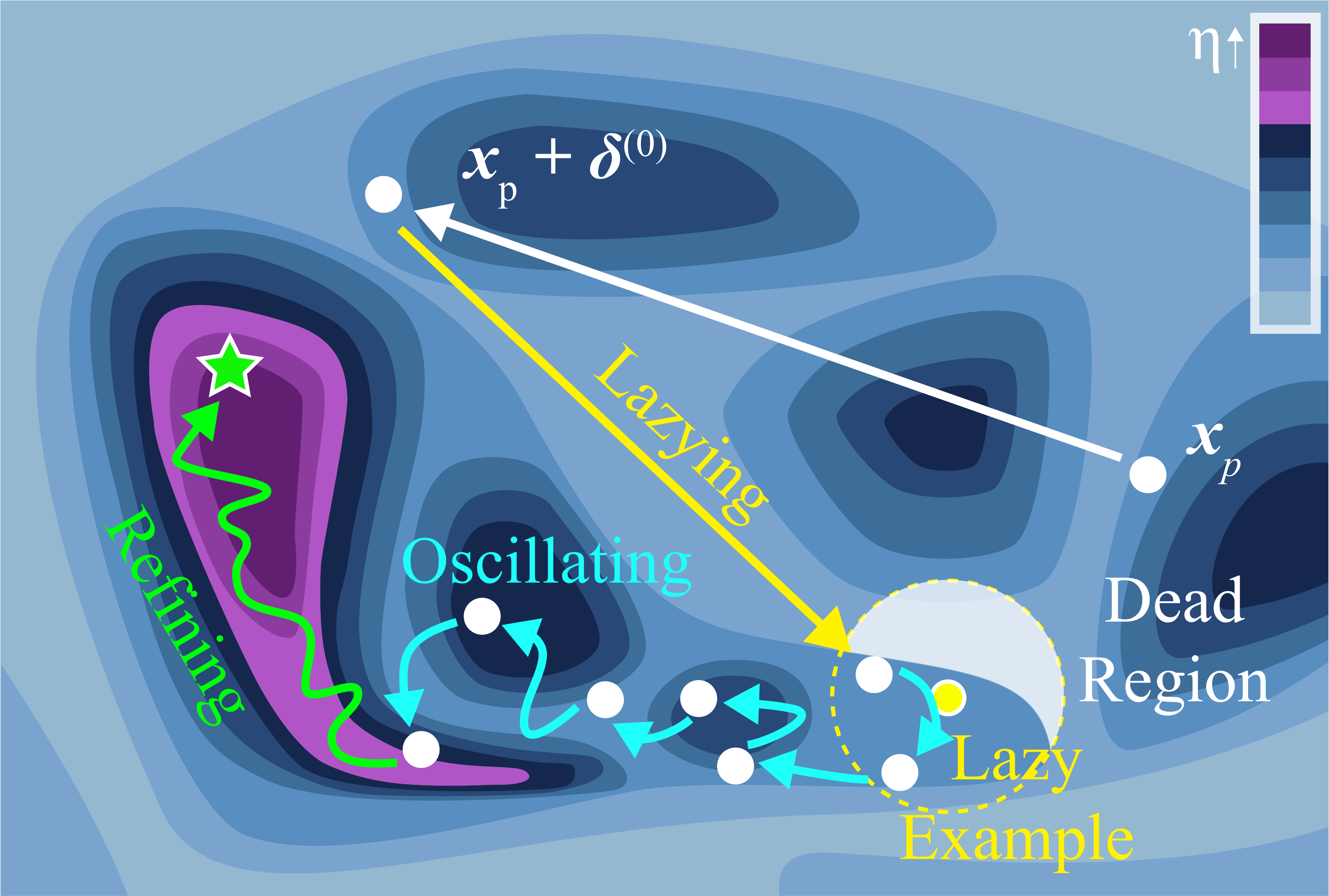}
\caption{Lazying–Oscillating–Refining illustrated via $\eta$ contours. Blue: Identity Region. Purple: Amplification Region.}
\label{fig:eta_map}
\vspace{-20pt}
\end{wrapfigure}

Section~\ref{section: Identity Region and Amplification Region} reveals that effective GSM examples reside in the amplification region, while most source images lie in the identity region. This suggests two critical phases for successful GSM: (1) driving the adversarial example from the identity region into the amplification region, and (2) identifying an appropriate operating point within the amplification region to amplify the perturbation toward $\boldsymbol{x}_{q}-\boldsymbol{x}_{p}$. We summarize successful GSM into three stages: \textcolor{darkblue}{Lazying–Oscillating–Refining} (Figure~\ref{fig:eta_map}).


\textbf{Stage I: Lazying.} In this stage, the adversarial example stays in the identity region and is driven toward the lazy example by Equation~\ref{approxiamation_gradient_2}. It rapidly saturates upon entering the neighborhood of the lazy example.


\textbf{Stage II: Oscillating.} In this stage, the example seeks to transition from the identity to the amplification region to further optimize Equation~\ref{GSM_single_pair_optimization_objective}, but this is not guaranteed. Entering the \textcolor{darkblue}{Active Region} near the lazy example in the Lazying stage enables the transition, whereas falling into the \textcolor{darkblue}{Dead Region} prevents it (empirical randomness).


\textbf{Stage III: Refining.} In this stage, the adversarial example is refined to find a suitable operating point within the amplification region. However, the updated example does not always stay within the amplification region, as large step sizes may push it out due to the narrowness of the region.

To quantify the three-stage progression and establish its empirical foundation, we design a suitable metric as an indicator. Based on our analysis, we identify: (1) in the Lazying stage, the adversarial example approaches the lazy example and eventually saturates, (2) in the Oscillating stage, it oscillates around the lazy example to search for the amplification region, and (3) in the Refining stage, it enters the amplification region and departs from the lazy example to locate a suitable operating point. 
Overall, the process exhibits a trajectory of rapidly approaching, then oscillating around, and finally departing from the lazy example. Motivated by this observation, we use similarity to the lazy perturbation as an indicator and define the \textcolor{darkblue}{Lazy Cosine Similarity (LCS)} to characterize this progression:
\begin{definition}
    \textbf{Lazy Cosine Similarity.} The Lazy Cosine Similarity of  $\boldsymbol{\delta}^{(t)}$ and $\boldsymbol{x}_{p}+\boldsymbol{\delta}^{(t)}$ is defined as the cosine similarity between the lazy perturbation $\Pi_{[-\epsilon,\epsilon]}(\boldsymbol{x}_{q}-\boldsymbol{x}_{p})$ and $\boldsymbol{\delta}^{(t)}$, expressed as:
    \begin{equation}
        \text{LCS}_{(\boldsymbol{x}_{p},\boldsymbol{x}_{q})}(\boldsymbol{\delta}^{(t)})=\frac{<\boldsymbol{\delta}^{(t)}, \Pi_{[-\epsilon,\epsilon]}(\boldsymbol{x}_{q}-\boldsymbol{x}_{p})>}{||\boldsymbol{\delta}^{(t)}||\cdot||\Pi_{[-\epsilon,\epsilon]}{\epsilon}(\boldsymbol{x}_{q}-\boldsymbol{x}_{p})||}.
    \end{equation}
\end{definition}

\subsection{Enabling $\ell_{\infty}$-Bounded High-Resolution GSM}
\label{section: PGD2_GSM}

Based on LCS, the three-stage Lazying–Oscillating–Refining progression becomes clearly visible, and the failure of $\ell_{\infty}$-bounded attacks in GSM can be naturally explained. For illustration, we perform GSM on LALIC~\cite{feng2025linear} with PGD at step sizes $\alpha$ of $0.01$, $0.005$, and $0.001$. The source and target images are identical to those in Figure~\ref{fig:teaser}(d). The perturbation budget $\epsilon$ is set to $0.08$, with $T=500$ iterations and the random seed of $2$. The resulting LCS trajectories over $500$ steps are shown in Figure~\ref{fig:LazyCosSim_illustration}(a), with the corresponding adversarial examples and reconstructions in Figures~\ref{fig:LazyCosSim_illustration}(b–d).

\begin{figure}
    \centering
    \includegraphics[width=1.00\linewidth]{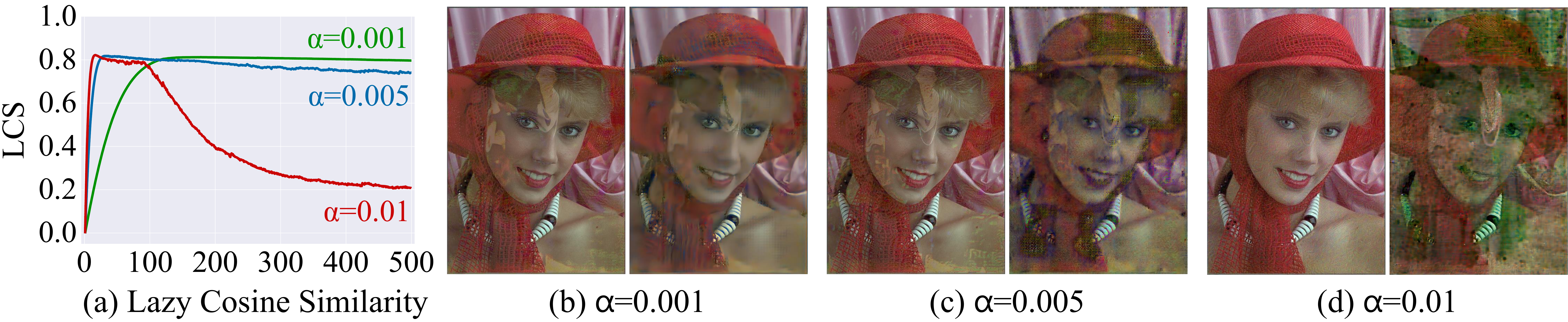}
    \caption{LCS trajectories and visualizations of PGD on LALIC under different $\alpha$, with $\epsilon = 0.08$.}
    \label{fig:LazyCosSim_illustration}
\end{figure}

As shown in Figure~\ref{fig:LazyCosSim_illustration}(a), with a large step size $\alpha=0.01$, the LCS rises rapidly and reaches an acme of $0.822$ at $T=17$, corresponding to the lazying stage. It then oscillates, remaining at $0.782$ at $T=93$, indicating strong similarity to the lazy example and thus the oscillating stage. Thereafter, the LCS drops sharply to $0.272$ and $0.207$ at $T=300$ and $T=500$, respectively, suggesting entry into the amplification region and identification of more favorable operating points, moving away from the lazy example with reduced reliance on identity gain. However, the fixed large $\alpha$ prevents further refinement in the amplification region, leading to failure at the refining stage (Figure~\ref{fig:LazyCosSim_illustration}(d)).

With a small step size $\alpha=0.005$, the LCS likewise rises rapidly, reaching $0.818$ at $T=30$. However, it then decays only marginally and remains $0.741$ at $T=500$, indicating persistent affinity to the lazy example, sustained reliance on identity gain (Figure~\ref{fig:LazyCosSim_illustration}(c)), and failure to enter the amplification region, thereby remaining in the oscillating stage throughout. According to Lemma~\ref{lemma:adv_in_amplification_region}, this configuration precludes the generation of effective GSM adversarial examples. This phenomenon becomes even more pronounced when $\alpha=0.001$ (Figure~\ref{fig:LazyCosSim_illustration}(b)), suggesting that small step size $\alpha$ impedes the transition into the amplification region and ultimately lead to failure.

Therefore, in the Oscillating stage, a large step size $\alpha$ is required to explore a wider region and accelerate entry into the amplification region. In contrast, in the Refining stage, a small step size is necessary to stably locate a suitable operating point. 
General $\ell_{\infty}$-bounded attacks fail due to their inability to accommodate both the Oscillating and Refining stages. 
A simple step-size strategy suffices to enable $\ell_{\infty}$-bounded attacks for GSM: 
\textcolor{darkblue}{Large $\alpha$ in the early phase to enter the amplification region, followed by smaller ones for stable refinement}. Notably, while such a strategy has been explored in prior work~\cite{croce2020reliable} for tasks such as classification attacks, it is typically \textit{optional}. In contrast, it is \textit{indispensable} for high-resolution GSM, as its absence leads to immediate failure (Section~\ref{section:Quantitative_Comparison_Reconstructions}).

\begin{wrapfigure}{r}{0.54\columnwidth}
\begin{minipage}{0.53\columnwidth}
\footnotesize
\vspace{-20pt}
\begin{algorithm}[H]
\caption{PGD$^{2}$-GSM}
\begin{algorithmic}[1]

\Require Source $\boldsymbol{x}_{p}$, target $\boldsymbol{x}_{q}$, victim LIC model $f$, perturbation budget $\epsilon$, steps $T$, initial step size $\alpha$, decay factor $k<1$, decay period $P$.

\State Initialize $\boldsymbol{\delta}^{(0)} \sim \text{Uniform}(-\epsilon, \epsilon)$
\For{$t = 0$ to $T-1$}
    
    \State $\alpha^{(t)} \gets \alpha\cdot k^{\lfloor\frac{t}{P}\rfloor}$ \Comment{\textcolor{darkgreen}{Periodic Geometric Decay}}
    \State $\boldsymbol{g}^{(t)} \gets -\nabla_{\boldsymbol{\delta}^{(t)}} ||f(\boldsymbol{x}_{p}+\boldsymbol{\delta}^{(t)}), \boldsymbol{x}_{q}||_{2}^{2}$
    \State $\boldsymbol{\delta}^{(t+1)} \gets \Pi_{[-\epsilon,\epsilon]}(\boldsymbol{\delta}^{(t)} + \alpha^{(t)}\cdot\mathrm{sgn}(\boldsymbol{g}^{(t)}))$
\EndFor
\State \Return $\Pi_{[0,1]}(\boldsymbol{x}_{p}+\boldsymbol{\delta}^{(T)})$
\end{algorithmic}
\label{algorithm:PGD2_GSM}
\end{algorithm}
\vspace{-20pt}
\end{minipage}
\end{wrapfigure}
To demonstrate the effectiveness of this strategy, we instantiate it with a simple step-size schedule, \textcolor{darkblue}{Periodic Geometric Decay}: the step size $\alpha$ is scaled by a decay factor $k$ every decay period $P$. To verify our approach, we integrate this schedule with PGD, yielding the minimal yet effective variant \textcolor{darkblue}{PGD$^{2}$-GSM} (Algorithm~\ref{algorithm:PGD2_GSM}). 
Section~\ref{section: experiments} will show that such a simple step-size scheduling suffices to enable $\ell_{\infty}$-bounded attacks for high-resolution GSM. This also highlights step-size control as a key mechanism underlying $\ell_{\infty}$-bounded high-resolution GSM.
\begin{figure}[t]
    \centering
    \includegraphics[width=0.98\linewidth]{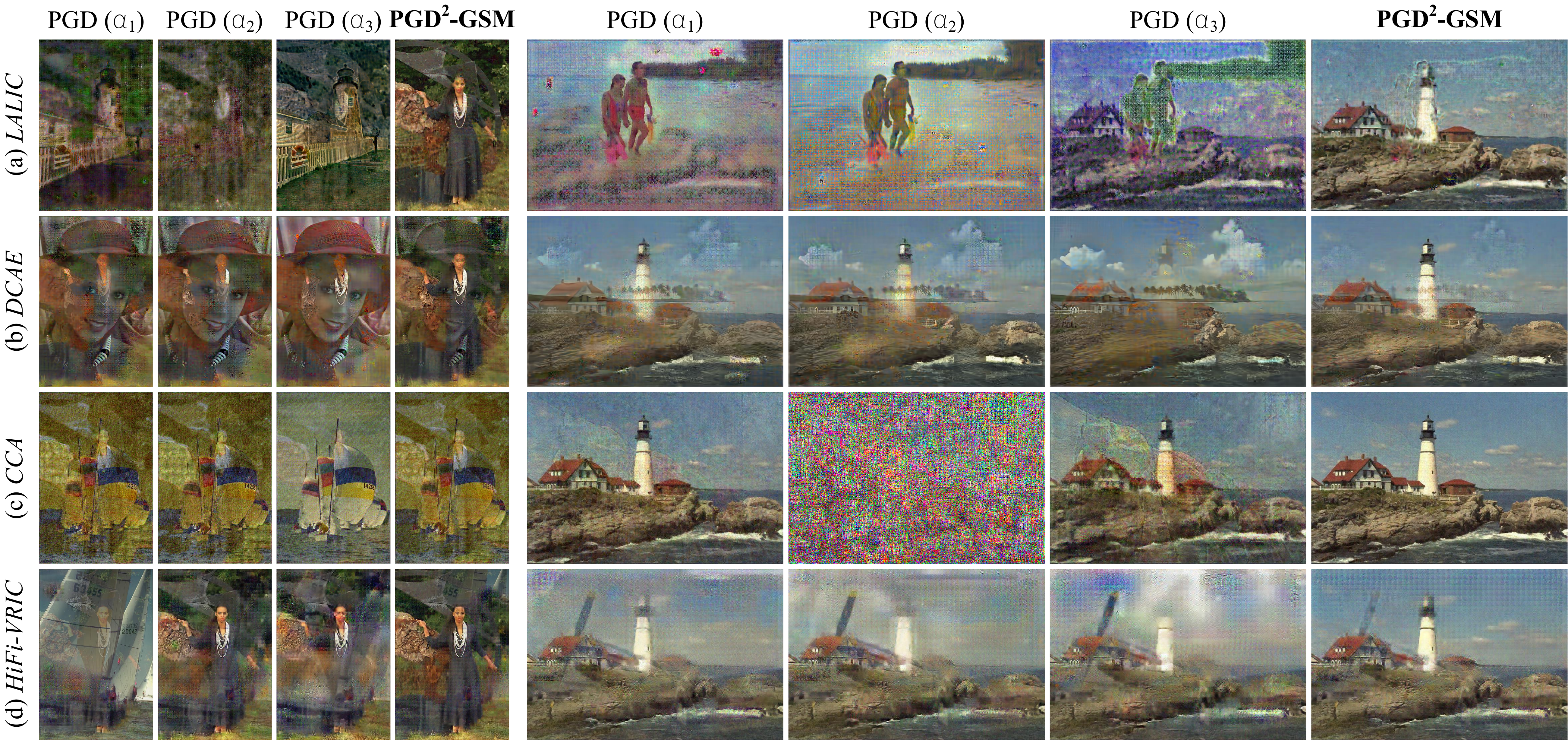}
    \caption{Visual comparison between PGD$^2$-GSM and PGD with different step sizes.}
    \label{fig:experiment_visualization_comparison}
\end{figure}

\section{Experiments and Results}
\label{section: experiments}
\subsection{Setup}
\label{section: Setup}
\textbf{LIC Benchmarks.} The LIC models used in our experiments include advanced \textit{LALIC}~\cite{feng2025linear} (CVPR'25), \textit{DCAE}~\cite{lu2025learned} (CVPR'25), and \textit{CCA}~\cite{han2024causal} (NeruIPS'24), as well as the early \textit{HiFi-VRIC}~\cite{cai2022high} (ACMMM'22). These models introduce innovations in transform architectures, entropy modeling, loss functions, and variable-rate modeling, respectively, providing a comprehensive evaluation.

\textbf{Dataset and Metrics.} We evaluate on the standard LIC benchmark Kodak~\cite{kodak1993} (Figure~\ref{fig:Kodak} in the Appendix), which contains 24 color images of size $512\times768$ or $768\times512$. Among them, \texttt{kodim18} and \texttt{kodim21} are employed as target images for the two resolutions, respectively, and the remaining 22 images are as sources. Peak Signal-to-Noise Ratio (PSNR), Multi-Scale Structural Similarity (MS-SSIM), Learned Perceptual Image Patch Similarity (LPIPS, AlexNet as the backbone), and Contrastive Language–Image Pretraining Similarity (CLIP, ViT-B/16 as the backbone) are used to evaluate the similarity between the manipulated reconstructions and the target images at the pixel, structural, and semantic levels, while bits per pixel (bpp) is adopted to measure the compression rate.

\textbf{Implementation Details.} For PGD$^{2}$-GSM, by default, the initial step size $\alpha$ is set to $0.01$ for \textit{LALIC}, \textit{DCAE}, and \textit{HiFi-VRIC}, and $0.004$ for \textit{CCA}. The decay factor $k$ and decay period $P$ are set to $0.5$ and $T/5$, respectively. All experiments for \textit{DCAE}, \textit{CCA}, and \textit{HiFi-VRIC} are conducted on a single NVIDIA A100 GPU (80GB). Due to hard-coded implementation constraints, \textit{LALIC} experiments are restricted to a single NVIDIA 2080 Ti GPU. We report results averaged over multiple random seeds.
\begin{figure}[t]
    \centering
    \includegraphics[width=0.96\linewidth]{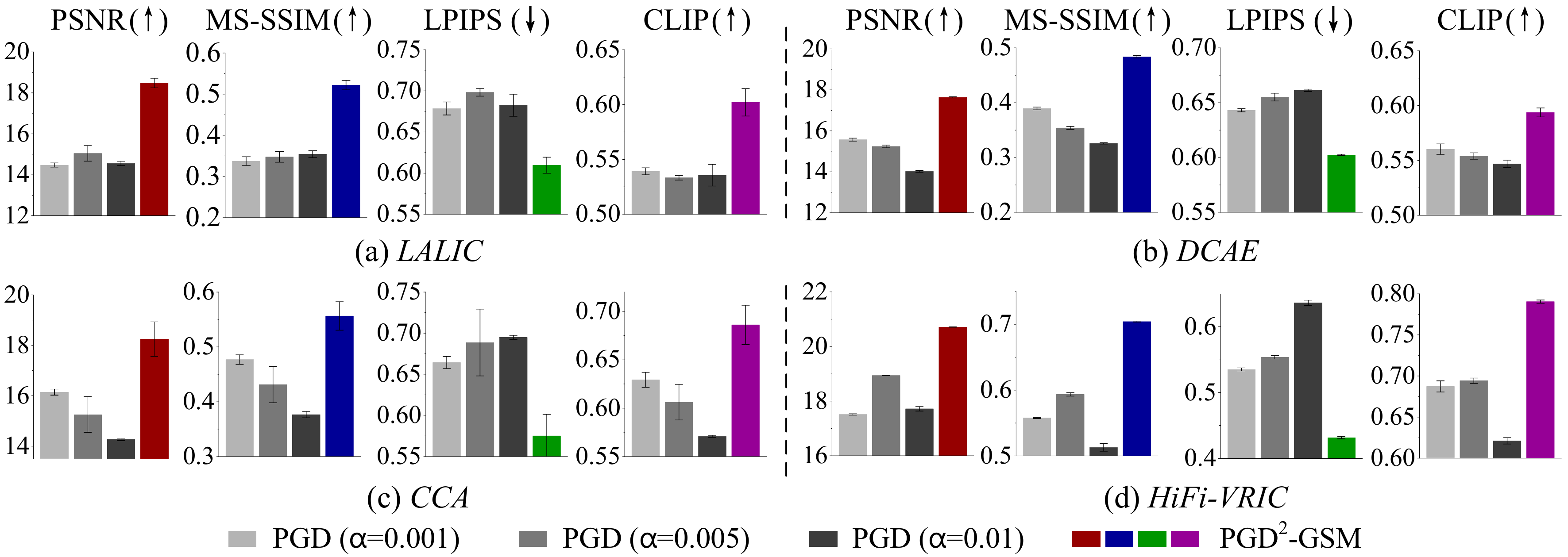}
    \caption{Quantitative comparison between PGD$^2$-GSM and PGD with different step sizes.}
    \label{fig:experiment_comparison_with_PGD}
\end{figure}

\begin{table}[h]
    \centering
    \caption{Compression rate (bpp) among clean images and GSM adversarial examples generated by PGD and PGD$^2$-GSM. N/A indicates that encoding fails to complete within a 300s time limit.}
    \begin{tabular}{c|ccccc}
        \toprule
        &Clean&PGD($\alpha_{1}$)&PGD($\alpha_{2}$)&PGD($\alpha_{3}$)&PGD$^{2}$-GSM\\
        \midrule
         \textit{HiFi-VRIC}~\cite{cai2022high}&1.030&N/A&N/A&N/A&N/A\\

         \textit{CCA}~\cite{han2024causal}&0.720&5.339$\pm$0.290&4.839$\pm$0.547&3.619$\pm$0.160&4.612$\pm$0.407\\

         \textit{DCAE}~\cite{lu2025learned}&0.106&0.528$\pm$0.004&0.485$\pm$0.006&0.363$\pm$0.005&0.986$\pm$0.021\\
         \textit{LALIC}~\cite{feng2025linear}&0.111&1.710$\pm$0.192&3.622$\pm$0.249&3.362$\pm$0.856&2.615$\pm$0.231\\
         
         \bottomrule
    \end{tabular}
    
    \label{tab:compression_rate_bpp}
\end{table}

\subsection{Quantitative Comparison of Semantic Manipulation in Reconstructions}
\label{section:Quantitative_Comparison_Reconstructions}
In this experiment, we quantitatively compare the proposed PGD$^{2}$-GSM and PGD with different step sizes $\alpha_{1}$, $\alpha_{2}$ and $\alpha_{3}$ in manipulating the reconstructed images of LIC. For PGD attacks, the step size is set to $\alpha_{1}=0.001, \alpha_{2}=0.005, \alpha_{3}=0.01$ for \textit{LALIC}, \textit{DCAE}, and \textit{HiFi-VRIC}, and to $\alpha_{1}=0.001, \alpha_{2}=0.002, \alpha_{3}=0.004$ for \textit{CCA}. We set $T=500, 40K, 40K$, and $5K$ for \textit{LALIC}, \textit{DCAE}, \textit{CCA}, and \textit{HiFi-VRIC}, respectively. Experiments are repeated with random seeds 0, 1, and 2, and the mean and standard deviation of PSNR, MS-SSIM, LPIPS, and CLIP are reported in Figure~\ref{fig:experiment_comparison_with_PGD}. The results show that PGD$^2$-GSM consistently outperforms PGD with different step sizes, yielding adversarial reconstructions that are closer to the target at the pixel, structural, and semantic levels.

\subsection{Visual Comparison of Semantic Manipulation in Reconstructions}
To visually demonstrate the significant advantage of PGD$^2$-GSM over PGD, we present the adversarial examples and their reconstructions in Figures~\ref{fig:adv_example_comparison_visualization} (in Appendix) and~\ref{fig:experiment_visualization_comparison}, respectively. The source images are \texttt{Kodim19} \& \texttt{Kodim12} for \textit{LALIC}, \texttt{Kodim04} \& \texttt{Kodim16} for \textit{DCAE}, \texttt{Kodim09} \& \texttt{Kodim03} for \textit{CCA}, and \texttt{Kodim10} \& \texttt{Kodim20} for \textit{HiFi-VRIC}, and the seed is 0. Figure~\ref{fig:experiment_visualization_comparison} shows that PGD with a fixed step size, due to its inability to balance the exploration of the amplification region with refinement, produces unstable adversarial reconstructions and poor visual manipulation across models. In contrast, PGD$^2$-GSM adjusts the step size via a simple periodic geometric decay scheme without any other modifications, yet it effectively balances the exploration of the amplification region and refinement, thereby achieving markedly stable and visually satisfactory GSM results across different models. This result demonstrates that PGD$^{2}$-GSM is the first to achieve stable high-resolution GSM.

\subsection{Impact of High-Resolution GSM on Compression Rate}
\label{section: Compression Rate}
\cite{chen2023toward} has shown that untargeted attacks on distortion $\mathcal{D}$ lead to a simultaneous collapse of the compression rate $\mathcal{R}$. We measure the bpp of adversarial examples generated by PGD and PGD$^2$-GSM in Section~\ref{section:Quantitative_Comparison_Reconstructions}, and observe that this phenomenon also occurs in high-resolution GSM. The results reported in Table~\ref{tab:compression_rate_bpp} indicate that the bpp of adversarial examples in high-resolution GSM is consistently much higher than that of clean images, and even prevents \textit{HiFi-VRIC} with an autoregressive entropy model from completing compression encoding within an reasonable time. This is because the perturbations refined within the amplification region exhibit highly intricate distributions.


\begin{figure}[t!]
    \centering
    \includegraphics[width=0.96\linewidth]{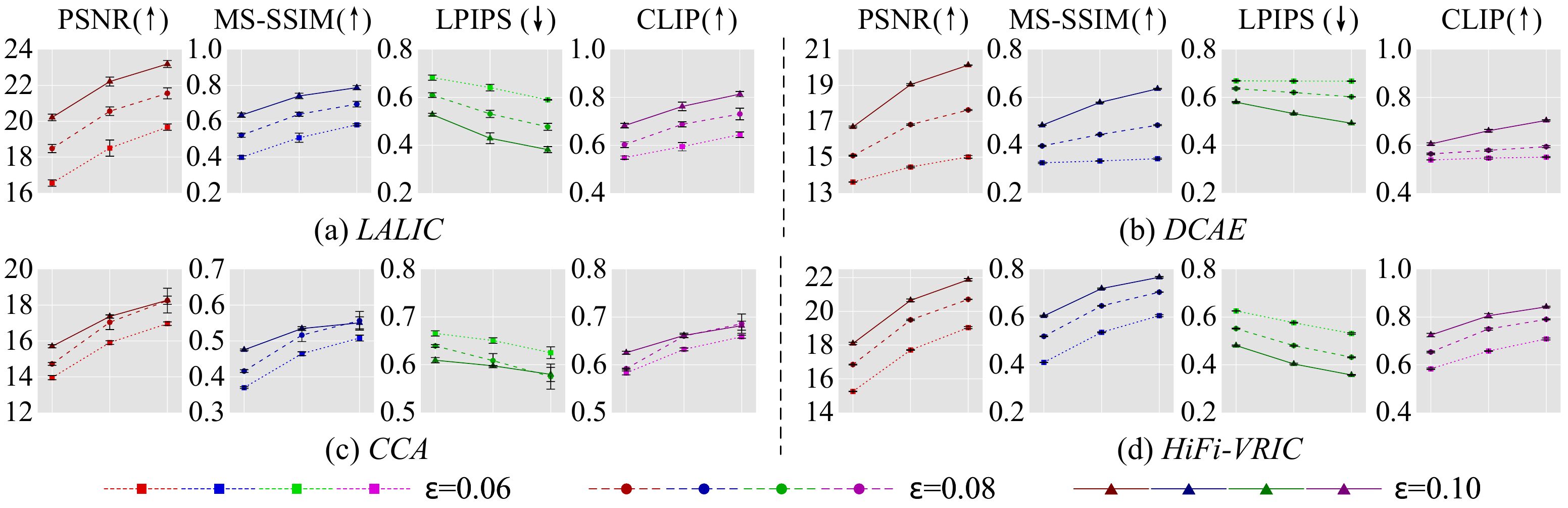}
    \caption{Quantitative results of PGD$^2$-GSM under different $T$ and $\epsilon$.}
    \label{fig:experiment_ablation_illustrations}
\end{figure}
\begin{figure}[t!]
    \centering
    \includegraphics[width=0.96\linewidth]{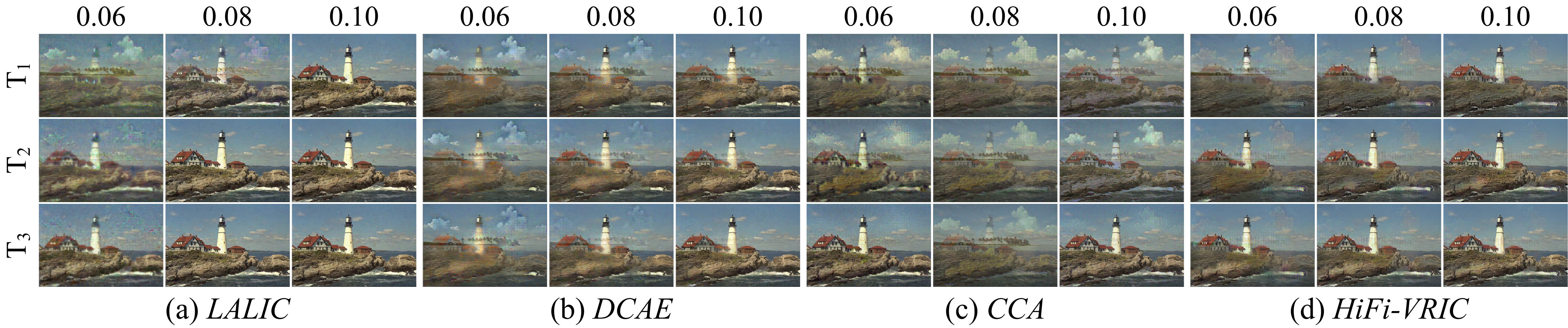}
    \caption{Visualization of PGD$^2$-GSM under different step numbers $T$ and perturbation budgets $\epsilon$.}
    \label{fig:ablation_visualizations}
\end{figure}
\subsection{Performance under Different Perturbation Budgets $\epsilon$}
\label{section: experiments_Different_Perturbation_Budgets}
The above experiments compare the performance of PGD$^{2}$-GSM and PGD under fixed $\epsilon=0.08$. However, the trade-off between stealthiness and the quality of the manipulated reconstruction may incentivize an adversary to either reduce or increase $\epsilon$. To this end, we further examine performance across varying $\epsilon=0.06,0.08$ and $0.10$, in conjunction with iteration numbers $T_{1}$, $T_{2}$, and $T_{3}$. Specifically, ($T_{1}$, $T_{2}$, $T_{3}$) are set to (500, 1000, 1500) for \textit{LALIC}, (5K, 20K, 40K) for \textit{DCAE} and \textit{CCA}, and (1K, 3K, 5K) for \textit{HiFi-VRIC}. Quantitative results are presented in Figure~\ref{fig:experiment_ablation_illustrations}, while qualitative visualizations are shown in Figure~\ref{fig:ablation_visualizations}. The results indicate that, under a fixed number of iterations, increasing $\epsilon$ effectively improves the quality of the manipulated reconstruction. In particular, at $T=T_{3}$ and $\epsilon=0.10$, the average PSNR reaches $23.2$ dB, $20.1$ dB, $18.3$ dB, and $21.9$ dB on \textit{LALIC}, \textit{DCAE}, \textit{CCA}, and \textit{HiFi-VRIC}, respectively. This suggests that, at this perturbation level, PGD²-GSM is capable of effectively steering the semantic structure of the reconstructed images.

\subsection{Performance under Different Iteration Numbers $T$}
\label{section: Performance under Different Iteration Numbers}
With $\epsilon$ fixed and $T$ treated as the variable of interest, Figures~\ref{fig:experiment_ablation_illustrations} and~\ref{fig:ablation_visualizations} illustrate the performance of PGD²-GSM across different $T$. Notably, unlike untargeted attacks, which typically succeed within tens to hundreds of iterations, high-resolution GSM generally requires thousands to tens of thousands of iterations to achieve satisfactory performance. This discrepancy arises because untargeted attacks only need to enter the amplification region, whereas high-resolution GSM demands additional refinement beyond that point.
Furthermore, results on \textit{DCAE} (Figure~\ref{fig:experiment_ablation_illustrations}(b)) reveal that when $\epsilon$ remains below a certain threshold, increasing the number of iterations yields little to no improvement. This highlights that both a sufficiently large $T$ and $\epsilon$ are crucial for effective GSM.

\section{Conclusions and Limitations}
\label{section: Conclusions and Limitations}
We present the first systematic study of $\ell_{\infty}$-bounded high-resolution GSM. Our analysis reveals a key distinction of LIC models: their intrinsic identity-mapping partitions the input space into the Identity Region and the Amplification Region. 
Successful GSM requires driving perturbations into the amplification region and refining them. Building on this, we characterize the dynamics by Lazying-Oscillating-Refining and introduce Lazy Cosine Similarity to validate this progression. 
This suggests the need for a large $\alpha$ to explore the amplification region and a smaller $\alpha$ for refinement. 
We show that general $\ell_{\infty}$-bounded attacks fail as their schedules cannot accommodate both the Oscillating and Refining stages. 
To this end, we propose the Periodic Geometric Decay schedule to enable $\ell_{\infty}$-bounded attacks on GSM, and integrate it with PGD, yielding the minimal yet effective variant, PGD$^{2}$-GSM. 
Nevertheless, our PGD$^{2}$-GSM is still sensitive to initialization and can be computationally expensive, with single-sample attacks approaching the cost of training a small model. Developing more robust initialization and improving efficiency remain important future directions.




\bibliographystyle{unsrt}
\bibliography{references}

\appendix
\newpage
{\centering \large \textbf{Appendix}\par}

\noindent\textbf{\large Section~\ref{section: Proofs}: Proofs}

Section~\ref{proof:adv_in_amplification_region}: Proof of Lemma~\ref{lemma:adv_in_amplification_region} \dotfill \pageref{proof:adv_in_amplification_region}

\bigskip

\noindent\textbf{\large Section~\ref{section: Supplementary Materials}: Supplementary Materials}

Section~\ref{section: Kodak Dataset}: Kodak Dataset \dotfill \pageref{section: Kodak Dataset}

\bigskip

\noindent\textbf{\large Section~\ref{section: Supplementary Experiments}: Supplementary Experiments}

Section~\ref{section: Visual Comparison of High-Resolution GSM Examples}: Visual Comparison of High-Resolution GSM Examples \dotfill \pageref{section: Visual Comparison of High-Resolution GSM Examples}

Section~\ref{section: GSM Example Visualization for Different Pertation Budgets and Iterations}: GSM Example Visualization for Different Pertation Budgets and Iterations \dotfill \pageref{section: GSM Example Visualization for Different Pertation Budgets and Iterations}

Section~\ref{section: Bypassing Defenses: A Case Study on JPEG-based Defenses}: Bypassing Defenses: A Case Study on JPEG-based Defenses \dotfill \pageref{section: Bypassing Defenses: A Case Study on JPEG-based Defenses}

Section~\ref{section: Ablations on Decay Factor $k$}: Ablations on Decay Factor $k$ \dotfill \pageref{section: Ablations on Decay Factor $k$}

\newpage
\section{Proofs}
\label{section: Proofs}
\subsection{Proof of Lemma~\ref{lemma:adv_in_amplification_region}}
\label{proof:adv_in_amplification_region}
\begin{proof}
    We prove Lemma~\ref{lemma:adv_in_amplification_region} by contradiction. 
    
    Assume that the adversarial example of GSM $\boldsymbol{x}_{p}+\boldsymbol{\delta}$ lies in identity region, then we have 
    \begin{equation}
        ||f(\boldsymbol{x}_{p}+\boldsymbol{\delta})-\boldsymbol{x}_{p}||\approx||\boldsymbol{x}_{q}-\boldsymbol{x}_{p}||,\quad\boldsymbol{x}_{p}+\boldsymbol{\delta}\in R_{I}.
        \label{equation_proof_1}
    \end{equation}
    By the definition~\ref{definition_identity_region}, we have 
    \begin{equation}
        ||f(\boldsymbol{x}_{p}+\boldsymbol{\delta})-\boldsymbol{x}_{p}||=||\boldsymbol{x}_{p}+\boldsymbol{\delta}+\boldsymbol{\eta}-\boldsymbol{x}_{p}||=||\boldsymbol{\delta}+\boldsymbol{\eta}||\ll||\boldsymbol{x}_{q}-\boldsymbol{x}_{p}||.
        \label{equation_proof_2}
    \end{equation}
    Equation~\ref{equation_proof_2} contradicts Equation~\ref{equation_proof_1}.
\end{proof}

\section{Supplementary Materials}
\label{section: Supplementary Materials}
\subsection{Kodak Dataset}
\label{section: Kodak Dataset}
The full Kodak dataset is shown in Figure~\ref{fig:Kodak} for reference.
\begin{figure}[h]
    \centering
    \includegraphics[width=0.90\linewidth]{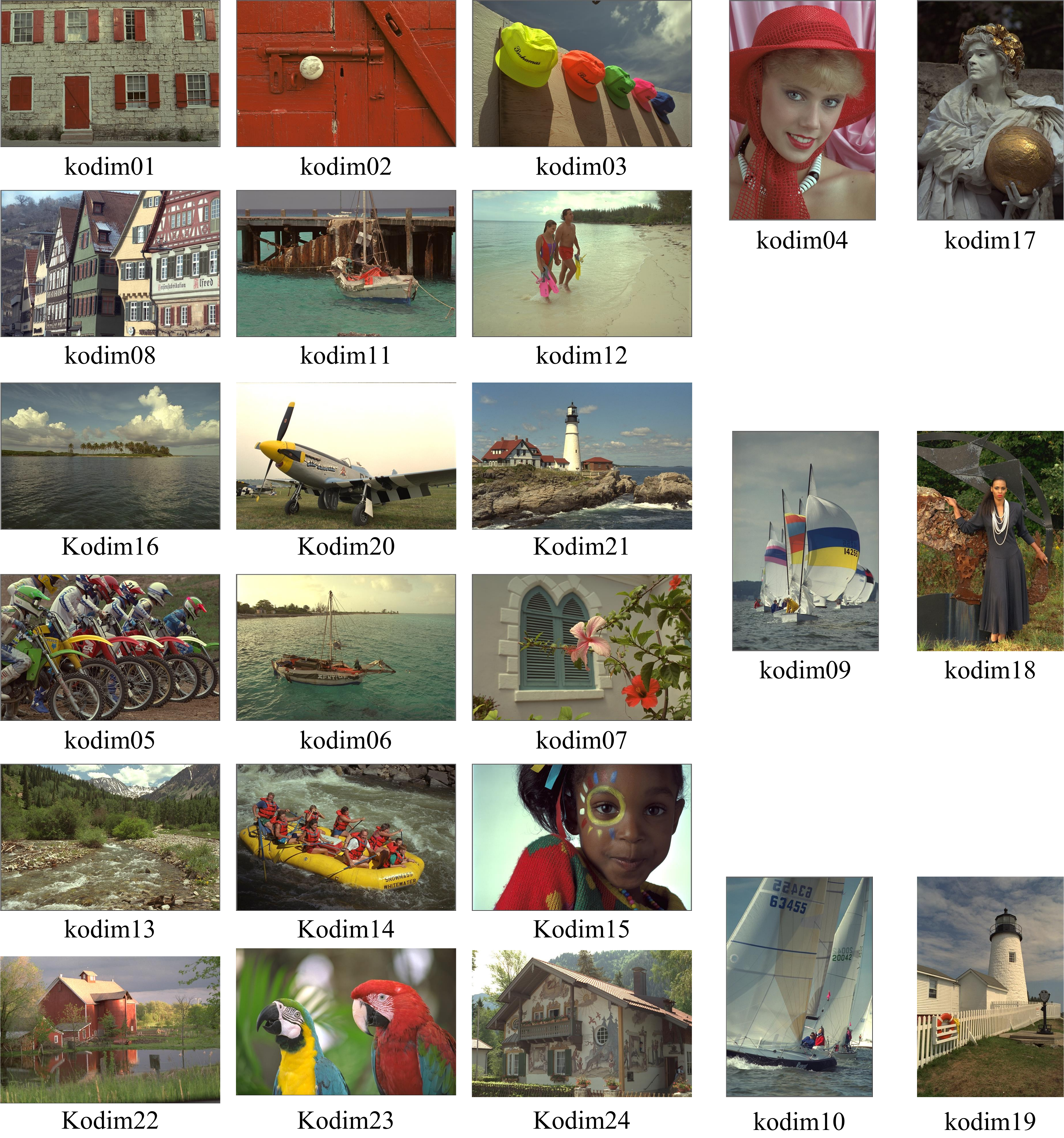}
    \caption{Kodak dataset for reference. Each image has a resolution of either $512\times768$ or $768\times512$. In this work, \texttt{kodim18} and \texttt{kodim21} are employed as target images for high-resolution GSM.}
    \label{fig:Kodak}
\end{figure}

\section{Supplementary Experiments}
\label{section: Supplementary Experiments}
\subsection{Visual Comparison of High-Resolution GSM Examples}
\label{section: Visual Comparison of High-Resolution GSM Examples}
We present visualizations of high-resolution GSM adversarial examples generated by PGD and PGD$^{2}$-GSM in Figure~\ref{fig:adv_example_comparison_visualization}, with the same parameter settings as in Section~\ref{section:Quantitative_Comparison_Reconstructions}. From the adversarial examples generated by PGD($\alpha_{1}$) on \textit{LALIC}, \textit{DCAE}, \textit{CCA}, and \textit{HiFi-VRIC}, we find that the perturbations resemble a watermark of the target image overlaid on the source image. This is because $\alpha_{1}$ is the smallest step size used in our experiments. As discussed in Section~\ref{section: PGD2_GSM}, PGD with a fixed small step size fails to enter the amplification region, causing the adversarial examples to remain in the identity region and approximate lazy examples. Such watermark-like adversarial examples closely resemble lazy examples. The watermark pattern diminishes as the step size increases to $\alpha_{2}$ and $\alpha_{3}$ or when using PGD$^{2}$-GSM. This is because larger step sizes enable broader exploration of the amplification region and faster entry into it, allowing the perturbations to leave the identity region and move away from lazy examples.
\begin{figure}[h]
    \centering
    \includegraphics[width=0.95\linewidth]{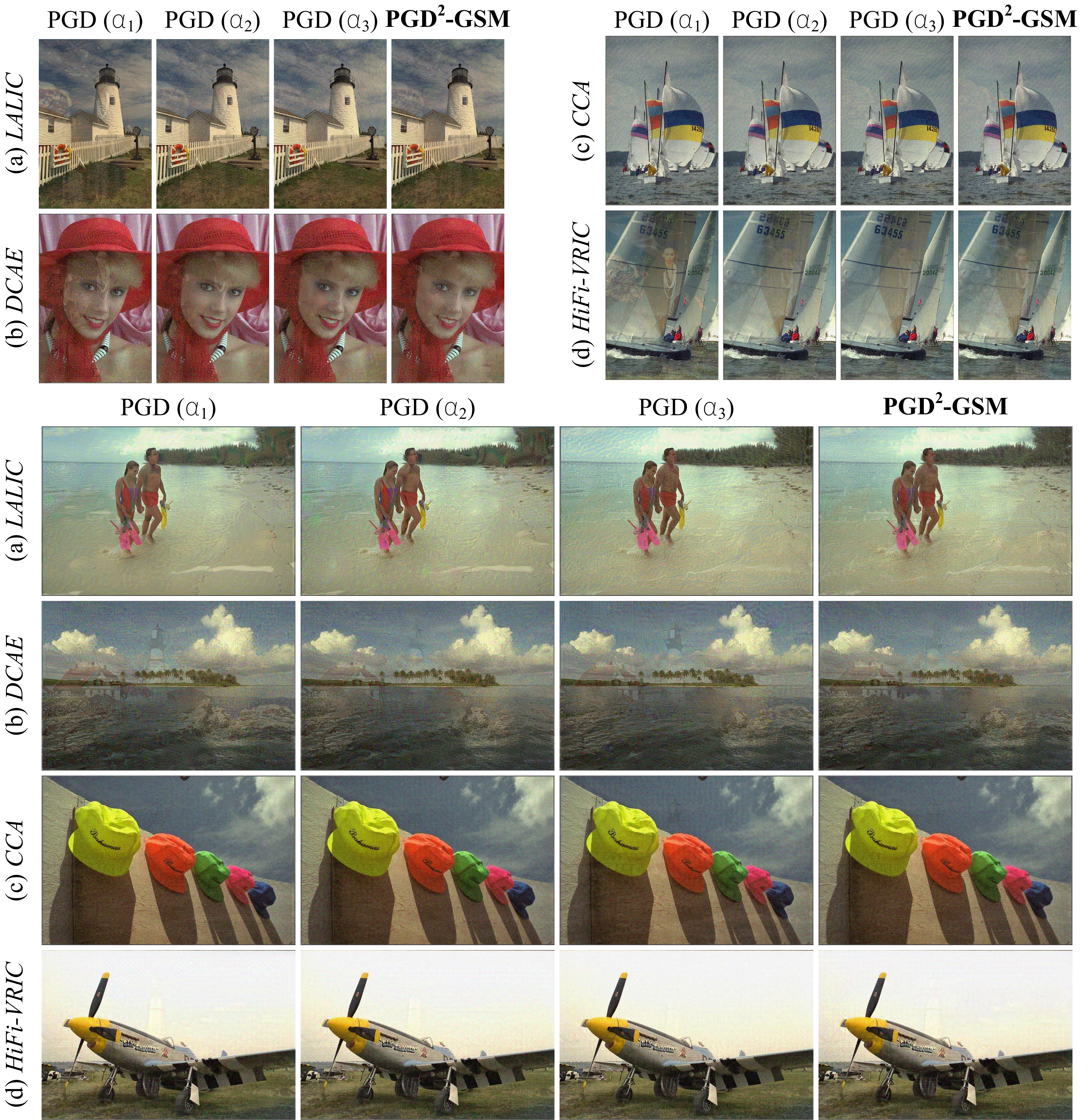}
    \caption{Visualizations of high-resolution GSM examples generated by PGD and PGD$^{2}$-GSM.}
    \label{fig:adv_example_comparison_visualization}
\end{figure}

\newpage
\subsection{GSM Example Visualization for Different Pertation Budgets and Iterations}
\label{section: GSM Example Visualization for Different Pertation Budgets and Iterations}
Following the settings in Section~\ref{section: experiments_Different_Perturbation_Budgets}, we present in Figure~\ref{fig:ablation_adv_examples_visualizations} visualizations of adversarial examples generated by PGD$^{2}$-GSM on \textit{LALIC}, \textit{DCAE}, \textit{CCA}, and \textit{HiFi-VRIC} under different perturbation budgets ($\epsilon=0.06, 0.08, 0.10$) and iteration numbers ($T=T_{1}, T_{2}, T_{3}$). The source and target images are \texttt{Kodim16} and \texttt{Kodim21}, respectively. Figure~\ref{fig:ablation_adv_examples_visualizations} reveals an interesting phenomenon: adversarial examples generated with smaller perturbation budgets $\epsilon$ and fewer iterations $T$ (e.g., $\epsilon=0.06$, $T=T_{1}$ in Figure~\ref{fig:ablation_adv_examples_visualizations}(a)) exhibit more pronounced watermark artifacts, and their manipulated reconstructions in Figure~\ref{fig:ablation_visualizations}(a) are of lower quality. In contrast, with larger $\epsilon$ and $T$ (e.g., $\epsilon=0.10$, $T=T_{3}$), the watermark artifacts become weaker or disappear, and the corresponding reconstruction quality improves. This indicates that escaping lazy examples and entering the amplification region requires sufficiently many iterations and adequate perturbation budget. Otherwise, the amplification region remains difficult to reach.
\begin{figure}[h]
    \centering
    \includegraphics[width=0.95\linewidth]{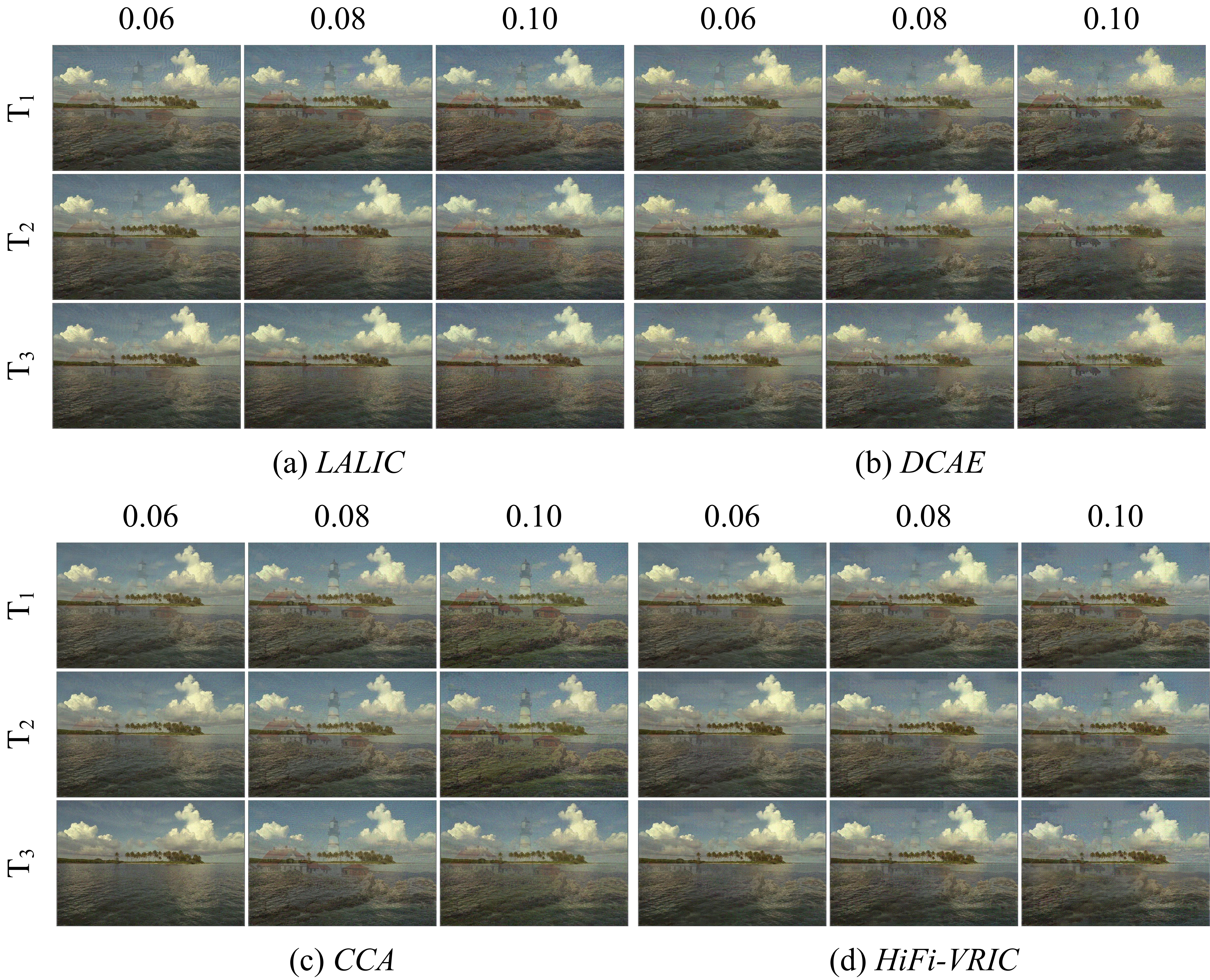}
    \caption{Visualizations of adversarial examples generated by PGD$^{2}$-GSM under different perturbation budgets and iteration numbers. The source and target images are \texttt{Kodim16} and \texttt{Kodim21}, respectively.}
    \label{fig:ablation_adv_examples_visualizations}
\end{figure}

\newpage
\subsection{Bypassing Defenses: A Case Study on JPEG-based Transformation Defense}
\label{section: Bypassing Defenses: A Case Study on JPEG-based Defenses}
We further evaluate high-resolution GSM under defensive settings in LIC models. Specifically, we consider JPEG-based transformation defense as a representative case to simulate practical defense scenarios. We set the perturbation budget of PGD$^{2}$-GSM to $0.10$. The number of iterations is set to $T=1500$ for \textit{LALIC}, $T=40K$ for \textit{DCAE} and \textit{CCA}, and $T=5K$ for \textit{HiFi-VRIC}. The initial step size is set to $0.01$ for \textit{LALIC}, \textit{DCAE}, and \textit{HiFi-VRIC}, and $0.004$ for \textit{CCA}. The random seed is $0$.

Under the setting without defenses in the LIC model, this configuration effectively enables high-resolution GSM, as discussed above. The corresponding reconstructions and quantitative results are presented in Figure~\ref{fig:JPEG_defense_visualizations} \textit{(w/o, w/o)} column and Table~\ref{table:JPEG_defense_quantizations} \textit{(w/o, w/o)} row, respectively. However, we find that the adversarial perturbations of high-resolution GSM are highly fine-grained and can be easily smoothed by preprocessing. With JPEG compression ($Q=90$) applied before reconstruction, high-resolution GSM completely fails, as shown in the \textit{(w/o, \textbf{with})} setting in Figure~\ref{fig:JPEG_defense_visualizations} and Table~\ref{table:JPEG_defense_quantizations}.

In practice, when a LIC model is deployed as a general compression standard, its implementation details are assumed to be publicly available, making the defense mechanism fully accessible to the attacker in a white-box setting. Therefore, the attacker can integrate differentiable preprocessing transformations into the LIC pipeline, enabling high-resolution GSM to bypass input preprocessing defenses. To this end, we incorporate a differentiable JPEG approximation (with $Q=90$) into the LIC model within high-resolution GSM, and present the reconstructions and quantization results in the \textit{(\textbf{with}, \textbf{with})} setting in Figure~\ref{fig:JPEG_defense_visualizations} and Table~\ref{table:JPEG_defense_quantizations}. As shown, on \textit{LALIC}, PGD$^{2}$-GSM effectively leverages the differentiable JPEG approximation to bypass the defense, albeit with some performance degradation. On \textit{HiFi-VRIC}, it is still able to marginally bypass the JPEG defense and achieve high-resolution GSM. In contrast, on \textit{DCAE} and \textit{CCA}, even with the differentiable JPEG approximation, PGD$^{2}$-GSM fails to effectively overcome the defense and achieve high-resolution GSM. This suggests that the difficulty of achieving high-resolution GSM varies across different models under both standard and defense settings, indicating that certain architectural factors play a significant role.

\begin{table}[h]
    \centering
    \setlength{\tabcolsep}{3.2pt}
    \caption{Performance of PGD$^{2}$-GSM for high-resolution GSM under standard and defense settings. In the \textit{Attack} column, \textit{w/o} and \textit{\textbf{with}} denote whether the differentiable JPEG approximation ($Q=90$) is used for generating adversarial examples. In the \textit{Inference} column, they denote whether JPEG ($Q=90$) is applied as a transformation defense during reconstruction.}
    \begin{tabular}{cc|cccc|cccc}
    \toprule
    
    \multirow{2}{*}{Attack}&\multirow{2}{*}{Inference}&\multicolumn{4}{c|}{\textit{LALIC}}&\multicolumn{4}{c}{\textit{DCAE}}\\
    \cmidrule{3-10}
    &&PSNR$\uparrow$&MS-SSIM$\uparrow$&LPIPS$\downarrow$&CLIP$\uparrow$&PSNR$\uparrow$&MS-SSIM$\uparrow$&LPIPS$\downarrow$&CLIP$\uparrow$ \\
    \midrule
         \textit{w/o}&\textit{w/o}&23.177&0.788&0.383&0.814&20.129&0.635&0.493&0.707 \\
         \textit{w/o}&\textit{\textbf{with}}&12.162&0.296&0.696&0.530&12.179&0.332&0.701&0.521 \\
         \cellcolor{gray!18}\textit{\textbf{with}}&\cellcolor{gray!18}\textit{\textbf{with}}&\cellcolor{gray!18}19.885&\cellcolor{gray!18}0.610&\cellcolor{gray!18}0.562&\cellcolor{gray!18}0.664&\cellcolor{gray!18}15.042&\cellcolor{gray!18}0.421&\cellcolor{gray!18}0.602&\cellcolor{gray!18}0.568 \\

    \midrule
    
    \multirow{2}{*}{Attack}&\multirow{2}{*}{Inference}&\multicolumn{4}{c|}{\textit{CCA}}&\multicolumn{4}{c}{\textit{HiFi-VRIC}}\\
    \cmidrule{3-10}
    &&PSNR$\uparrow$&MS-SSIM$\uparrow$&LPIPS$\downarrow$&CLIP$\uparrow$&PSNR$\uparrow$&MS-SSIM$\uparrow$&LPIPS$\downarrow$&CLIP$\uparrow$ \\
    \midrule
         \textit{w/o}&\textit{w/o}&18.078&0.552&0.587&0.684&21.946&.0.7692&0.357&0.842 \\
         \textit{w/o}&\textit{\textbf{with}}&12.414&0.319&0.713&0.578&11.919&0.325&0.776&0.565 \\
         \cellcolor{gray!18}\textit{\textbf{with}}&\cellcolor{gray!18}\textit{\textbf{with}}&\cellcolor{gray!18}14.003&\cellcolor{gray!18}0.400&\cellcolor{gray!18}0.639&\cellcolor{gray!18}0.591&\cellcolor{gray!18}16.961&\cellcolor{gray!18}0.575&\cellcolor{gray!18}0.499&\cellcolor{gray!18}0.671 \\
         
    \bottomrule
    \end{tabular}
    \label{table:JPEG_defense_quantizations}
\end{table}

\begin{figure}[h]
    \centering
    \includegraphics[width=1.00\linewidth]{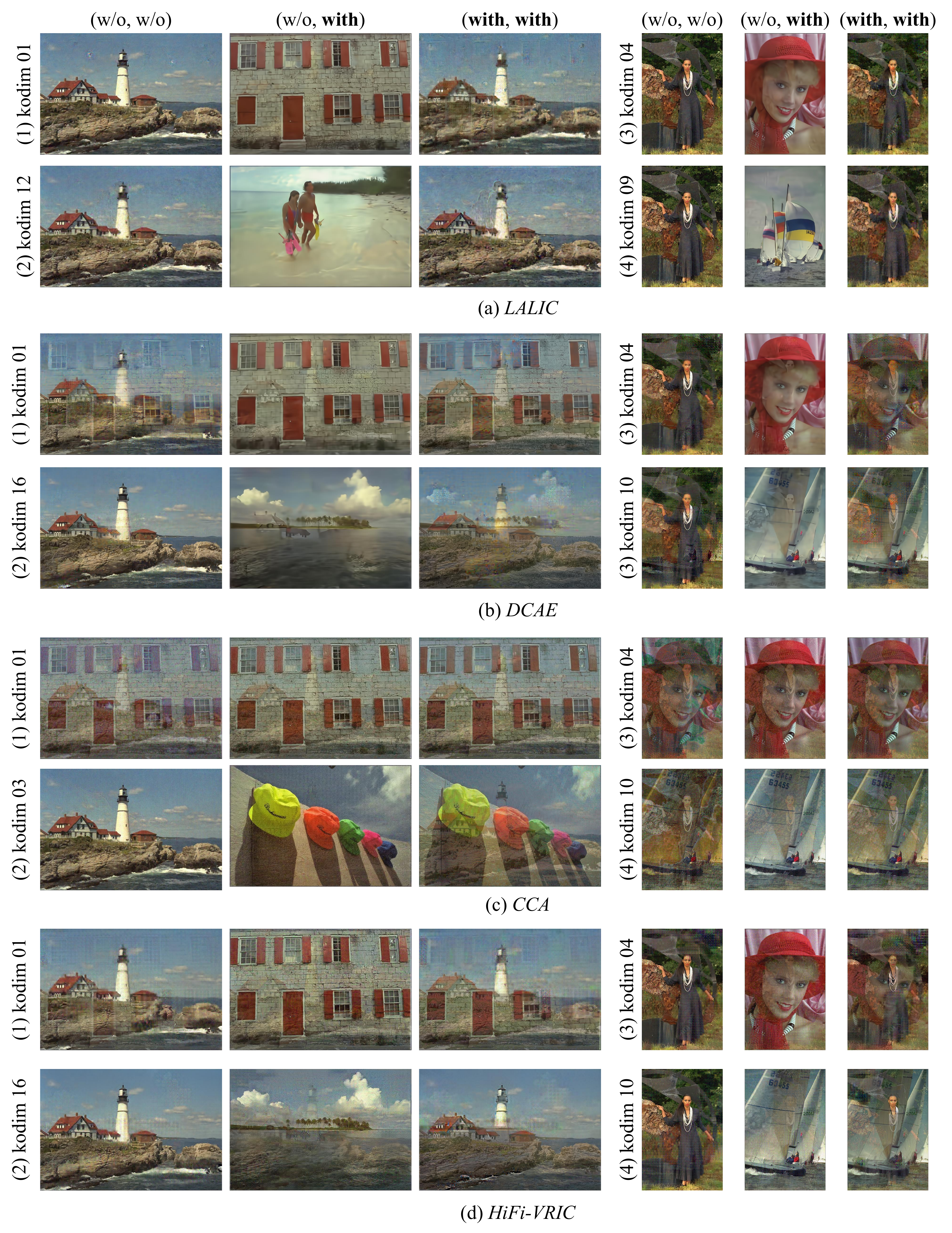}
    \caption{Reconstructions of high-resolution GSM examples. In the parentheses, the first \textit{w/o} denotes that high-resolution GSM examples are generated directly on the LIC model, while \textit{\textbf{with}} indicates generation on the cascaded differentiable JPEG and LIC model. The second \textit{w/o} denotes reconstruction without JPEG, whereas \textit{\textbf{with}} applies a JPEG transformation prior to reconstruction. The source image is indicated on the left of each image group. The target image is either \texttt{kodim18} or \texttt{kodim21}, depending on the resolution.}
    \label{fig:JPEG_defense_visualizations}
\end{figure}

\clearpage
\subsection{Ablations on Decay Factor $k$}
\label{section: Ablations on Decay Factor $k$}
In this experiment, we investigate the impact of the decay factor $k$ in the Periodic Geometric Decay step-size schedule on the performance of high-resolution GSM. The number of iterations, initial step size, and random seed are kept consistent with those in Section~\ref{section: Bypassing Defenses: A Case Study on JPEG-based Defenses}. We vary the decay factor as $k \in \{1.0, 1.5, 2.0, 2.5, 3.0\}$. Notably, when $k=1.0$, PGD$^{2}$-GSM degenerates to PGD. Quantitative results are reported in Table~\ref{table: ablations_on_k}. As shown in Table~\ref{table: ablations_on_k}, with an interval of $\Delta k=0.5$, the performance of PGD$^{2}$-GSM consistently peaks at $k=2.0$ in \textit{LALIC}, \textit{DCAE}, and \textit{HiFi-VRIC}, and peaks at $k=2.5$ in \textit{CCA}. A small $k$ results in insufficient decay, causing the optimization to spend excessive iterations in the oscillating stage exploring amplification regions while under-refining. Conversely, a large $k$ leads to overly aggressive decay, limiting exploration and potentially preventing the adversarial example from effectively entering amplification regions. The decay factor $k$ plays a critical role in balancing exploration (Oscillating stage) and refinement (Refining stage).

\begin{table}[h]
    \centering
    \setlength{\tabcolsep}{3.4pt}
    \caption{Ablation study of PGD$^{2}$-GSM on the decay factor $k$.}
    \begin{tabular}{c|cccc|cccc}
    \toprule
    
    \multirow{2}{*}{$k$}&\multicolumn{4}{c|}{\textit{LALIC}}&\multicolumn{4}{c}{\textit{DCAE}}\\
    \cmidrule{2-9}
    &PSNR$\uparrow$&MS-SSIM$\uparrow$&LPIPS$\downarrow$&CLIP$\uparrow$&PSNR$\uparrow$&MS-SSIM$\uparrow$&LPIPS$\downarrow$&CLIP$\uparrow$ \\
    \midrule
         $k=1.0$&14.654&0.360&0.689&0.546&13.966&0.325&0.661&0.548 \\
         $k=1.5$&18.389&0.511&0.626&0.608&17.117&0.441&0.623&0.580 \\
         $k=2.0$&\textbf{18.740}&\textbf{0.535}&\textbf{0.596}&\textbf{0.616}&\textbf{17.652}&\textbf{0.484}&\textbf{0.602}&\textbf{0.598} \\
         $k=2.5$&18.249&0.511&0.619&0.606&17.404&0.477&0.607&0.591 \\
         $k=3.0$&17.863&0.489&0.633&0.585&16.840&0.455&0.616&0.579 \\

    \midrule
    
    \multirow{2}{*}{$k$}&\multicolumn{4}{c|}{\textit{CCA}}&\multicolumn{4}{c}{\textit{HiFi-VRIC}}\\
    \cmidrule{2-9}
    &PSNR$\uparrow$&MS-SSIM$\uparrow$&LPIPS$\downarrow$&CLIP$\uparrow$&PSNR$\uparrow$&MS-SSIM$\uparrow$&LPIPS$\downarrow$&CLIP$\uparrow$ \\
    \midrule
         $k=1.0$&14.271&0.378&0.694&0.572&17.811&0.517&0.635&0.625 \\
         $k=1.5$&\textbf{17.489}&0.515&0.622&0.654&20.684&0.694&0.449&0.788 \\
         $k=2.0$&17.462&0.527&0.605&0.663&\textbf{20.719}&\textbf{0.705}&\textbf{0.432}&\textbf{0.791} \\
         $k=2.5$&17.462&\textbf{0.535}&\textbf{0.602}&\textbf{0.679}&20.511&0.697&0.436&0.788 \\
         $k=3.0$&17.156&0.524&0.605&0.660&20.246&0.685&0.447&0.778 \\
         
    \bottomrule
    \end{tabular}
    \label{table: ablations_on_k}
\end{table}



\end{document}